\documentclass[submission,copyright,creativecommons]{eptcs}
\providecommand{\event}{ICLP 2026} 

\usepackage{iftex}

\newcommand{\M}{\mathcal M}
\newcommand{\N}{\mathcal N}
\newcommand{\U}{\mathcal U}
\newcommand{\V}{\mathcal V}

\newcommand{\Z}{\mathcal Z}
\newcommand{\W}{\mathcal W}
\renewcommand{\P}{\mathcal P}

\newcommand{\assign}{ \xleftarrow{}}
\newcommand{\true}{\textsc{True}}
\newcommand{\false}{\textsc{False}}

\newcommand{\ran}{\rangle}
\newcommand{\lan}{\langle}

\newcommand{\lang}{\mathcal{E}\mathcal{S}}

\newcommand{\exec}{\mathit{exec}}
\newcommand{\filter}{\mathit{Filter}}

\newcommand{\poss}{\mathit{Poss}}

\newcommand{\pickup}{\mathit{pickup}}
\newcommand{\drop}{\mathit{drop}}
\newcommand{\quench}{\mathit{quench}}
\newcommand{\repair}{\mathit{repair}}
\newcommand{\Holding}{\mathit{Holding}}
\newcommand{\Broken}{\mathit{Broken}}
\newcommand{\Fragile}{\mathit{Fragile}}

\usepackage{amsthm}
\usepackage{amsmath}
\usepackage{amssymb}
\newcommand*{\myvec}[1]{\overrightarrow{\mkern0mu#1}}

\newtheorem{theorem}{Theorem}

\theoremstyle{definition}
\newtheorem{definition}{Definition}
\newtheorem{example}{Example}

\ifpdf
  \usepackage{underscore}         
  \usepackage[T1]{fontenc}        
\else
  \usepackage{breakurl}           
\fi

\title{A Counterfactual Cause in Situation Calculus}

\author{Daxin Liu
\institute{State Key Laboratory for Novel Software Technology \&\\
School of Artificial Intelligence
\\ Nanjing University, China}
\email{daxin.liu@nju.edu.cn}
\and
Vaishak Belle
\institute{School of Informatics\\
The University of Edinburgh, UK}
\email{vbelle@ed.ac.uk}
}
\def\titlerunning{A Counterfactual Cause in Situation Calculus}
\def\authorrunning{D. Liu \& V. Belle}

\begin{document}
\maketitle

\begin{abstract}
Perhaps the most popular modern formulation of actual causality is the HP account by Halpern and Pearl. Recent advancement has focused on extension of HP account to lift its limited expressiveness, in particular, Batusov and Soutchanski proposed a notion of actual achievement cause in the situation calculus, a rich first-order formalism of actions and changes. Among other things, the first-order nature allows for determining the cause of quantified effects in a given action history therein. While intuitively appealing, Batusov and Soutchanski's account is not defined in a counterfactual perspective. In this paper, we propose a notion of cause based on counterfactual analysis. In the context of action history, we show that our notion of cause generalizes naturally to a notion of achievement cause. We analyze the relationship between our notion of the achievement cause and the achievement cause by Batusov and Soutchanski. Finally, we relate our account of cause to Halpern and Pearl's account of actual causality. Particularly, we note some nuances in applying a counterfactual viewpoint to disjunctive effects, a common thorn in definitions of actual causes.
\end{abstract}

\section{Introduction}

The topic of causality \cite{pearl2009causality} has been widely argued to play a central role in artificial intelligence, mainly because it determines and captures how the agent understands the world it operates in. Philosophers and machine learning people have argued for the need for \emph{type causality}, also referred to as \emph{general causality}, where we might be interested in knowing things like whether smoking causes cancer or asthma. But philosophers have also discussed the need for \emph{actual causality} \cite{halpern2016actual}, sometimes also called \emph{token causality}, which describes what sequence of events leads to something happening. E.g., in a legal setting, we wish to understand if the perpetrator intended to cause harm and whether the perpetrator carried out the harm to the victim or an accidental fire in the victim's home.

As argued by Halpern and Hitchcock \cite{halpern2011actual}, the topic of actual causation is a subject of intense study but also considerable disagreement because examples that make a case for one cause against the other can give authenticity or validity to one formalism versus the other. To a large extent, the study of actual causality in AI can be attributed to the importance of actual causality as an essential ingredient of the larger study of causation, primarily resting on David Hume's account of \emph{but-for} cause, i.e. \emph{counterfactual} cause
\cite{hume1748}. Essentially, this says that had event $A$ not happened, then event $B$ would not happen or would not be true. Such a formulation does not require that $A$ is the unique cause for $B$. This is exemplified by the so-called structural equation model approach, which provides a rigorous mathematical framework for looking at actual causation by Pearl \cite{pearl1998definition,pearl2000models}. Since the original definition by Pearl, a number of modifications to the structural model account of actual causation have been introduced, perhaps most recently by Halpern and Pearl's account (henceforth HP account) \cite{halpern2001causes,halpern2015modification,halpern2016actual}.

Causation, and especially actual causation, remains an important topic for capturing blameworthiness and harm \cite{chockler2004responsibility}. Although the structural equation model provides an attractive and arguably simple formalism for describing variables and their descendants, its simplicity becomes a limitation as it lacks temporal reasoning and quantifiers, and fails to distinguish actions from effects. Hence, it is difficult to work with once we go beyond atomic situations requiring one-step analysis. This motivated Pearl and his collaborators \cite{hopkins2007causality} to revisit the structural equation modeling approach in the language of the situation calculus, one of the most widely studied languages for reasoning about action and change. However, these too had problems as identified by Soutchanski \cite{batusov2018situation}  and led to providing a fresh and recent account of actual causality in the situation calculus.

Their key idea is to provide a definition of ``achievement cause'' which, given a history of actions and an outcome, tries to use regression to identify the sequence that leads to this event being true. This has been expanded to the epistemic context \cite{khan2021knowing}. Although this is a very worthy start, what we ask in this paper is whether it is possible to characterize actual causality simply in terms of counterfactual analysis. We argue that a simple definition that almost completely lifts the ``but-for'' definition of causality, but in a rich action setting, is possible. Essentially, our definition looks very much like a reading of the definition of causality in some sense, closely following David Lewis's intuition \cite{Lewis1973}, but also to some extent following the intuition of the Halpern-Pearl modified definition of structural equation models for actual causation \cite{halpern2015modification}.

In fact, motivated by lifting the limited expressiveness of the HP account and adapting the elegant counterfactual analysis therein, Halpern and Pass \cite{HalpernPass2025} recently proposed to define abstract causes for frameworks where counterfactuals are defined: 
\begin{quote}
   ...as long as we start with a framework that has a definition of counterfactual that a user is comfortable with and supports these constructs, we can essentially import the HP definition to that framework.
\end{quote}
Hence, essentially, our approach follows this idea and adapts it to the framework of situation calculus action theories. 

The rest of the paper is organized as follows: in the next section, we present a modal logic that allows us to express action and changes. Based on it, we introduce our account of the actual cause in actions and analyze its application across various scenarios in Section~\ref{sec:cause}. We formally connect our account to Batusov and Soutchanski's work in Section~\ref{sec:bscause}. In Section~\ref{sec:hp}, we relate our account to the HP account and analyze the differences and limitations of our account. Finally, we discuss related works in Section \ref{sec:relatedwork} and conclude in Section \ref{sec:conclusion}.

We do not think our account is strictly better than others. Nevertheless, since situation calculus offers a natural way to model actions, effects, and causation, our approach’s simplicity could benefit applications like Golog-based robot programming and future action languages requiring causal reasoning.

\section{A Modal Logic of Action and Change}

We start with a modal variant of the situation calculus, the logic $\lang$ \cite{lakemeyer2005semantics}. Our choice of language may seem unusual, but it is shown in \cite{lakemeyer2005semantics} that this modal language corresponds to a major fragment of the classical situation calculus with a more elegant treatment of situations while maintaining many merits of the situation calculus, like the basic action theories to specify the dynamics of a domain and regression to reason about actions. There are, of course, many works explicating the links between the situation calculus and logic programming; see \cite{reiter2001knowledge} for starters.

The logic features a fixed countable domain called \emph{standard names} $\N$, which amounts to having an infinite domain closure axiom together with the unique name assumption.

\paragraph{Syntax} There are terms of sort \emph{object} and \emph{action}. Variables of sort object are denoted by symbols $x, y, \ldots,$ and of sort action by $a$. For standard names, we have $\N =\N_A \cup \N_O$ where $\N_A$ and $\N_O$ are countably infinite sets of action and object standard names. We assume object standard names look like constants. To represent actions, we have \emph{rigid} function symbols along with arities. \footnote{
``Rigid'' means the interpretation is fixed; this is in comparison with ``fluent'', which means the interpretation might vary according to actions.
} For example, $\pickup(C)$ and $\drop(C)$ represent the actions of picking up object $C$ and dropping object $C$.
\emph{Ground terms} are terms without \emph{variables}. A term or formula is \emph{primitive} if its parameters are object standard names.
For example, $\pickup(C)$ and $\drop(C)$ are ground primitive terms. We assume the set of action standard names $\N_A$ is just the set of primitive ground action terms. 

Formulas are built using \emph{fluent predicate} symbols (with arities), equality, the usual logical connectives, and quantifiers, i.e., $\land, \neg, \forall$. They can also be constructed by modal operators $[\cdot],\Box$. Namely, formulas $[t]\phi$ and $\Box\phi$ are read as ``$\phi$ holds after action $t$'' and ``$\phi$ holds after any action sequence''. Namely, a formula $\phi$ could be:
$$
t_1 = t_2 \mid F(t_1,\ldots, t_k) \mid  \phi_1 \land \phi_2  \mid \neg \phi \mid \forall x. \phi \mid [t] \phi \mid \Box\phi
$$
where $t,t_1,\ldots,t_k$ are terms, $F$ is a fluent predicate symbol, $\phi_1$ and $\phi_2$ are other formulas.

For action sequence $z = a_1 \cdots a_k$, we write $[z] \phi$ to mean $[a_1]\ldots[a_k]\phi$. The logic also includes a special fluent $\poss(a)$ to express that executing action $a$ is possible. A \emph{sentence} is a formula without free variables. A formula is called \emph{static} if it does not mention $[\cdot]$ and $\Box$. We use $\phi^x_n$ to denote the formula obtained by substituting all free occurrences of $x$ by $n$.

\paragraph{Semantics}
The semantics is given in terms of possible worlds where a world $w$ determines what holds initially and after any action sequences. Formally, a world $w$ is a mapping $w: \P_F \times \Z \mapsto \{0,1\}$ where $\P_F$ is the set of all primitive formulas, i.e. formulas of the form $F(n_1,\ldots n_k)$ (here, $F$ is a $k$-ary fluent and $n_i$ are standard names), and $\Z$ is the set of action sequences $\Z = \N^*_A$ (including the empty sequence $\lan \ran$) to $\{0,1\}$.
Let $\W$ be the set of all worlds.

\begin{definition}[Truth of Formulas] 
Given a world $w \in \W$ and a sentence $\phi$, we define $w \models \phi$ as $w,\lan \ran \models \phi$, where for any $z \in \Z$, $t \in \N_A$:
\begin{itemize}
    \item $w,z \models F(n_1,\ldots,n_k)$ iff $w[F(n_1,\ldots, n_k),z]=1$;

    \item $w,z \models n_1=n_2$ iff $n_1$ and $n_2$ are identical;

    \item $w,z \models \phi_1 \land \phi_2$ iff $w,z \models \phi_1$ and $w,z \models \phi_2$;

    \item $w,z \models \neg \phi$ iff $w,z \nvDash \phi$;

    \item $w,z \models \forall x.\phi$ iff $w,z \models \phi^x_{n}$ for all standard names $n$ of the right sort;

    \item $w,z \models [t] \phi$ iff $w,z\cdot t \models \phi$;

    \item $w,z \models \Box \phi$ iff $w,z\cdot z' \models \phi$ for all $z' \in \mathcal{Z}$;
\end{itemize}
\end{definition}

We say a formula $\phi$ is \emph{satisfiable} if there exists $w,z$ s.t. $w,z \models \phi$ and \emph{valid} if for all $w,z$, it holds $w,z \models \phi$.  \emph{Logical entailment} is defined as usual, i.e. $\Sigma \models \phi$ if for all $w,z$ s.t. $w,z \models \Sigma$, $w,z \models \phi$.

\paragraph{Basic Action Theory} To express the dynamic of a domain, $\lang$ use a variant of the \emph{basic action theory} (BAT) which includes:
\begin{itemize}
    \item $\Sigma_0$: the initial state axiom is a finite set of static sentences describing what holds initially, also called the initial knowledge base (KB);

    \item $\Sigma_{ap}$: the action preconditions axioms describing the preconditions for actions to be executable;

    \item $\Sigma_{post}$: the successor state axioms, one for each fluent, describing how actions change the truth of the fluent, incorporating Reiter's solution \cite{reiter2001knowledge} to the frame problem.     
\end{itemize}

We lump these axioms as $\Sigma$, i.e.  $\Sigma= \Sigma_0 \cup \Sigma_{ap} \cup \Sigma_{post}$. The following is a BAT for the block world domain. 

\begin{example}
\label{example:blockdomain}
    Consider a simple blocks world domain, involving picking up and dropping objects,
but also quenching (rapidly cooling to very low temperatures) objects so that they become fragile,
adapted from \cite{khan2021knowing}. The action $\pickup(x)$ is always possible (the robot's hand has infinite space), and dropping is only possible when it is already holding the object. Also, broken objects in the robot’s hand can be repaired. So, a $\Sigma_{ap}$ could be \footnote{Free variables are implicitly universally quantified from the outside. The $\Box$ modality has lower syntactic precedence than the connectives,
and $[\cdot]$ has the highest priority.}
$$
\begin{aligned}
\Box \poss(a) \equiv&  (a = \pickup(x))  
\vee 
(a = \drop(x) \land \Holding(x)) \\
\vee & (a = \quench(x) \land \Holding(x)) \\
\vee & (a = \repair(x) \land \Holding(x) \land \Broken(x)).
\end{aligned}
$$
Suppose further that $\Holding$ might be affected by the action $\pickup$ and $\drop$ in the literal way. $\drop$ a fragile object might cause it to be broken and $\repair$ makes it no longer being broken. $\quench$ an object makes it fragile. These can be captured by the following successor state axioms $\Sigma_{post}$:
$$
\begin{aligned}
 &\Box [a]\Holding(x)\equiv a = \pickup(x) \vee a\neq \drop(x) \land \Holding(x)\\
& \Box [a] \Broken(x)  \equiv a=\drop(x) \land \Fragile(x)  \vee a \neq \repair(x) \land \Broken(x)\\
& \Box [a] \Fragile(x)  \equiv a= \quench(x) \vee \Fragile(x) 
\end{aligned}
$$

Let us suppose that $\Sigma= \Sigma_0 \cup \Sigma_{ap} \cup \Sigma_{post}$ where $\Sigma_0$ is as 
$$
\left\{
\Fragile(C),\neg \Fragile(D), \forall x. \neg \Holding(x),\neg \Broken(C),\neg \Broken(D)
\right\}, 
$$
then the BAT $\Sigma$ entails the following:
    \begin{itemize}
        \item $\Sigma \models [\pickup(C)][\drop(C)] \Broken(C)$

        \item $\Sigma \models [\pickup(D)][\quench(D)] \Fragile(D)$  \qed
    \end{itemize}

\end{example}
        
\paragraph{Projection by Regression}
An important task in reasoning about action is projection, which is to decide what holds after an action. More formally, given a BAT $\Sigma$ deciding if a query $\phi$ holds after some action sequence $z$, that is, deciding if $\Sigma \models [z]\phi$. We want to ensure that the action sequence is executable (aka valid and/or legal). 
So let $\exec(\lan \ran) = \true$, and $\exec(a \cdot z) = \poss(a) \land [a]\exec(z)$. Then the projection problem is to determine if
\begin{equation}
    \label{eq:proj}
\Sigma \models \exec(z) \land [z]\phi.
\end{equation}
One of the projection solutions is by regression. Regression systematically converts the query about the future, $[z]\phi$, into another logical equivalent one about now $\phi'$. Namely, the regression $\mathfrak{R}[\Sigma,z,\phi]$ 
of $\phi$ wrt $z$ under $\Sigma$ is the weakest precondition of $[z]\phi$ being true under $\Sigma$. If the context is clear, we also write $\mathfrak{R}[z,\phi]$ instead of $\mathfrak{R}[\Sigma,z,\phi]$. $\lang$ adapts a similar solution of regression as in \cite{reiter2001knowledge}, which recursively replaces fluents in the query about one-step future with their corresponding right-hand side (RHS) of the successor state axioms, which eventually generates a formula without the $[\cdot]$ modality. Hence, the projection problem is reduced to checking if 
\begin{equation}
    \label{eq:regress}
\Sigma_0 \models \mathfrak{R}[\Sigma,\lan\ran,\exec(z) \land [z]\phi],
\end{equation}
where classical first-order reasoning applies. For example, to determine if the BAT $\Sigma$ in Example \ref{example:blockdomain} entails a query $[drop(C)]\Broken(C)$, the regression $\mathfrak{R}[drop(C),\Broken(C)]$ replaces $\Broken(C)$ with the RHS of the successor state of $\Broken$ and instantiates it with the ground action $\drop(C)$, resulting in a formula $\Fragile(C) \vee \Broken(C)$. Namely, after the action $\drop(C)$, $C$ is broken if and only if $C$ is fragile or $C$ is already broken before the action.

\section{Counterfactual Achievement Cause}
\label{sec:cause}

Batusov and Soutchanski in 2018 \cite{batusov2018situation} provided an account of \emph{actual achievement cause} in the situation calculus, overcoming many problems that previous works \cite{hopkins2007causality} had in lifting the expressiveness of HP causality. Later, 
Khan and Soutchanski in 2020 \cite{khan2020necessary} 
studied the sufficient and necessary conditions for actual achievement cause and provided an account that coincides with Batusov and Soutchanski's definition but incorporates the notion of counterfactual analysis, an important aspect in actual causality \cite{halpern2016actual}. In 2021, Khan and Lesp{\'e}rance \cite{khan2021knowing} reconsidered the above definition and provided an account of actual achievement cause in the epistemic situation calculus in terms of modalities, among other things, which allows one to reason about knowledge about the cause, i.e., causal knowledge. A drawback that is common in the above proposal is that causality is defined recursively, and inferring the cause in a causal setting is intuitively non-trivial. The reason is that their notion of cause takes many aspects into account, such as \emph{counterfactual}, \emph{preemption},  and \emph{direct or indirect cause}. Here we show that if we drop some of these features, we are still able to provide a simple account of counterfactual cause in the logic $\lang$.

\paragraph{Minimal Cause} We start with the notion of minimal cause, which is based on the notion of counterfactual.
\begin{definition}[Minimal Cause]
\label{def:minimalcause}
Given a BAT $\Sigma$ and a static sentence $\phi$ representing the goal, we say an action sequence $z$ is the \emph{minimal cause} of $\phi$ wrt $\Sigma$ if
\begin{enumerate}
    \item $\Sigma \models \neg \phi$;
    \item $\Sigma \models [z] \phi$;
    \item $z$ is minimal.
\end{enumerate}


\end{definition}

Clearly, this definition requires a notion of \emph{order} for action sequence $\Z$ or a notion of pairwise distance among action sequences (so that one minimizes the relative distance to the empty sequence $\lan \ran$). Below, we provide several definitions of causality in terms of different notions of minimality \cite{belle2023counterfactual}.

\begin{definition}[Length-based Minimality]
Given two sequences $z$,$z'$, define length-based minimality as minimizing $|(length(z)-length(z'))|$, which is the absolute value of the difference in lengths. Length is defined
inductively: $length(\lan \ran) = 0$, and $length(z \cdot a) = length(z) +1$.    
\end{definition}
Intuitively, this criterion measures the number of actions needed for a goal from the initial state.

\begin{example}
    Consider the BAT $\Sigma$ in Example \ref{example:blockdomain}, suppose the goal is $\Broken(C)$, the length-based minimal cause of $\phi$ is $z= \pickup(C) \cdot \drop(C)$. That is:
    \begin{enumerate}
        \item $\Sigma \models \neg  \Broken(C)$; 
        \item $\Sigma \models [z] \Broken(C)$;
        \item $z$ in the least action sequence to achieve $\Broken(C)$.     \qed
    \end{enumerate}

\end{example}

Causes defined in this way might not be unique. Consider the goal $\phi:=\exists x. \Holding(x)$ in the above example, both actions $\pickup(C)$ and $\pickup(D)$ are minimal causes of $\phi$ in $\Sigma$. Another possible measure is based on the properties of the world that are affected. For any $z \in \Z$, define $fluents(z)$ as the set of all fluents mentioned on the left-hand side of the successor state axioms where actions in $z$ occur.

\begin{definition}[Fluent-based Minimality]
Given two sequences $z$, $z'$, define fluent-based minimality as minimizing $|size(fluents(z))-size(fluents(z'))|$.    
\end{definition}
This criterion measures the affected properties of the world, i.e., fluents, between the initial state and goal states.

\begin{example}
    Consider the BAT $\Sigma$ in Example \ref{example:blockdomain} again, suppose the goal is $\phi:=\Broken(C) \vee \Broken(D)$. The action sequence $z = \pickup(C) \cdot \drop(C)$ has fluent set $\{\Holding(x),\Broken(x)\}$. Alternatively, the action sequence $z'=\pickup(D) \cdot \quench(D) \cdot \drop(D)$ can achieve the goal as well (one has to quench $D$ as $D$ is not fragile), yet with a fluent set $\{\Holding(x),\Broken(x),\Fragile(x)\}$ which is larger. So $z$ is a cause of $\phi$  in $\Sigma$ instead of $z'$.  \qed
\end{example}

Minimizing the fluent sets of action sequences might be problematic, as irrelevant ground actions could be included.

\begin{definition}[Plan-and-Effect Minimality]
Given two sequences $z$,$z'$, define plan-and-effect minimality as minimizing both length-based and fluent-based minimality (say, in a lexicographic way). 
\end{definition}

\begin{example}
   Let $\Sigma$ be as in Example \ref{example:blockdomain} again, and the goal be as $\phi:=\Broken(D)$. Action sequence $z = \pickup(D) \cdot  \quench(D) \cdot \drop(D)$ has the same fluent set as the action sequence $z' = \pickup(C) \land \drop(C) \cdot \pickup(D) \cdot  \quench(D) \cdot \drop(D) $, namely, $\{\Holding(x),\Broken(x),$ $\Fragile(x)\}$ (note that we are considering fluent predicates, not fluent instances, so the affected fluents predicates are the same).
    As $z$ is minimal in length, $z$ is the cause of $\phi$ in $\Sigma$ instead of $z'$. 
    \qed
\end{example}

We note that these variant notions of minimal causes align with the logic of \emph{counterfactual dependence}. In the actual world (the basic action theory $\Sigma$), the goal $\phi$ is initially false ($\neg \phi$ holds). The causal nature is revealed by considering the counterfactual: had the action sequence not occurred, the goal would remain unachieved by any smaller sequence. Thus, these minimal causes correspond to minimal plans that transition the system from $\neg\phi$ to $\phi$.
One could also draw comparisons to the HP definition, the original HP definition of a counterfactual cause: 
where the negation of the formula (the negation of the goal) is set to hold. We call an intervention a cause when the negation of the formula holds, the intervention leads to the formula being true, and this intervention is minimal.

Another remark is that these notions of minimal cause are defined in terms of all possible action sequences. In certain scenarios, one might already know the action history, i.e., the so-called \emph{narratives}, and wish to find the exact actions in the history that cause the goal. This is exactly what Batusov and Soutchanski's account \cite{batusov2018situation} and its variants have done. Here, we propose an alternative and hopefully simpler definition of achievement cause for such a setting.

\paragraph{Minimal Cause in Narratives}
By a \emph{causal setting} $\mathcal{C}$, we mean a triple of BAT $\Sigma$, action sequence $z$, and goal $\phi$, i.e. $\mathcal{C} = \lan \Sigma, z, \phi \ran$ such that $\Sigma \models \exec(z) \land [z] \phi$. Usually, the BAT $\Sigma$ is fixed, and we write $\mathcal{C}=\lan z, \phi \ran$ for short.

By $z' \subseteq z$, we mean $z'$ is a prefix subsequence of $z$, namely, $z' \subseteq z$ iff there exists $z''$ such that $ z = z' \cdot z''$. By disallowing $z''$ to be empty, we obtain the relation $z' \subset z$. 

Given a causal setting $\mathcal{C} = \lan z,\phi \ran$, we define the \emph{minimal achievement cause} (or simply \emph{achievement cause}) as the minimal prefix $z' \subseteq z$ that leads to $\phi$ under $\Sigma$. For this, we need to consider the counterfactual of $z'$ being absent, namely, the sequence $z''= z \backslash z'$ obtained from $z$ by removing the prefix $z'$. However, simply checking if $z''$ leads to $\phi$ under $\Sigma$ is insufficient, as $z''$ might contain actions that are now inexecutable due to the removal of $z'$ from $z$. We wish to filter these actions as well. We denote the result by 
$\filter(z\backslash z')$ and the function $\filter(\cdot)$ is recursively defined as:
\begin{enumerate}
    \item $\filter(\lan \ran) := \lan \ran$; 
    \item for any prefix $z''' \subseteq z\backslash z'$, wlog, assume $z''' = z^\star \cdot a$, then 

    \begin{enumerate}
        \item if $ \Sigma \models [\filter(z^\star)]\poss(a)$, $\filter(z''')= \filter(z^\star) \cdot a$; 
        \item otherwise $\filter(z''') = \filter(z^\star)$.
    \end{enumerate}
    
\end{enumerate}


E.g., consider the action sequence $z = \pickup(C) \cdot \drop(C) \cdot \pickup(D) \cdot \quench (D) \cdot \drop(D)$ in Example \ref{example:blockdomain}, if we remove the prefix $z' = \pickup(C)$ in $\Sigma$, the remaining executable action sequence of $z\backslash z'$ is $\filter(z\backslash z')= \pickup(D) \cdot \quench (D) \cdot \drop(D)$ as removing $\pickup(C)$ makes $\drop(C)$ impossible in $\Sigma$.

\begin{definition}[Achievement Cause for Narratives]
\label{def:acn}
Given a causal setting $\mathcal{C}= \lan z, \phi\ran$, such that $\Sigma \models \neg \phi$ and $\Sigma \models [z]\phi$. We call a prefix sequence $z'$ of an action sequence $z$, i.e. $z' \subseteq z$, a cause under $\mathcal{C}$ if 
\begin{enumerate}

    \item $\Sigma \models [\filter(z\backslash z')] \neg \phi$;

    \item $\Sigma \models [z''] \phi$ for all $z''$ such that $z' \subseteq z'' \subseteq z$;

    \item no subsequence of $z'$ holds for Items 1 and 2.   
\end{enumerate}    
\end{definition}
Item 1 is a \emph{sufficient} condition in terms of counterfactual, while Item 2 is a \emph{necessary} condition. Item 3 is a minimal condition (in the sense of sequence length), namely, we are interested in the minimal prefix of $z$ that satisfies Items 1 and 2. Note that, due to Item 3, the cause defined in this way is unique. Intuitively, the subsequence $z'$ is a cause of $\phi$ under BAT $\Sigma$ and narrative $z$, if it is the minimal subsequence that, after executing it, $\phi$ always holds (Item 1), and in the counterfactual that it is absent, the remaining legal actions will not change the truth of $\phi$ (Item 2). We show by examples that our notion of achievement cause is reasonable and useful.

\begin{example}
\label{example:causeinactions}
 Consider the BAT $\Sigma$ in Example \ref{example:blockdomain}. Suppose the goal is $\phi_1:=\Broken(C)$. Clearly, for the narrative  
$z = \pickup(C) \cdot \drop(C) \cdot \pickup(D)$, we have $\Sigma \models \exec(z) \land [z] \phi_1$ and the achievement cause here is the prefix $z' = \pickup(C) \cdot \drop(C)$: in the contingency of $z'$ being absent, the action $\pickup(D)$ alone would not cause $\Broken(C)$ under $\Sigma$. Hence, our notion of cause can successfully identify redundant actions in a narrative that are irrelevant to the achievement of a goal. \qed
\end{example}

\begin{example}
\label{example:disjunctivecauseinactions}
Now, consider the BAT $\Sigma$ in Example \ref{example:blockdomain} again, but this time, we consider another \emph{disjunctive} goal $\phi_2:=\Holding(C) \vee \Holding(D)$, for the narrative $z = \pickup(C) \cdot \pickup(D)$, we have $\Sigma \models \exec(z) \land [z] \phi_2$ and the achievement cause here is the prefix $z' = \pickup(C) \cdot \pickup(D)$: in the contingency of $z'$ being absent, $\filter(z\backslash z') = \lan \ran$, hence under which $\phi_2$ would not hold. The action $\pickup(C)$ is \emph{not} a cause, as even if it is absent, the remaining sequence $\pickup(D)$ makes $\Holding(D)$ true, causing the goal $\phi_2$ to hold. \qed

 Admittedly, some might claim that this is counterintuitive as it is ultimately the action $\pickup(C)$ that achieves the goal. In fact, this is because counterfactual causation suffers in handling preemption (when one event achieves an effect before another can). Our notion of the cause will disregard the temporal order of occurrences of multiple atomic competing events and include all the competing events as a cause. This aligns with the latest version of the HP account of actual cause \cite{halpern2015modification} where the joint of these events is identified as the cause and each individual event is called \emph{part of the cause}. We come back to this in Section \ref{sec:hp}.
\end{example}

\paragraph{Complexity of Reasoning about Minimal Causes} Before ending this section, we briefly touch on the reasoning complexity of the various minimal causes above. The problem is, in general, undecidable, even for the case of simple length-based minimality. This is because, as mentioned before, the minimal causes above correspond to the minimal plan; hence, reasoning about cause subsumes the planning problem in the situation calculus as a special case. The latter is known to be undecidable. In fact, to do planning, one needs to solve the projection problem as in Eq~\eqref{eq:proj}, which requires first-order reasoning as in Eq~\eqref{eq:regress}, hence, it is undecidable. Nevertheless, in practice, one might wish to consider decidable fragments such as $C^2$, the two-variable fragment with counting quantifiers \cite{PacholskiST00}, or to impose constraints on the knowledge base $\Sigma_0$, say, to have complete knowledge, so that one could have a Prolog interpreter of the language and use Prolog for a reasoning engine as in \cite{reiter2001knowledge}.

A more interesting case is for the minimal cause in narratives. Of course, the problem is still in general undecidable, yet the question we ask here is what is the complexity in terms of the size of the narratives if we have an oracle that can solve the projection problem in $O(1)$ time. For a narrative $z$, let $|z|$ be the number of actions in $z$. Naively, to reason about the achievement cause for a narrative $z$, one needs to check all the prefixes ($|z|$ in total) of $z$, and for each such prefix $z'$, one needs to check items 1 and 2 in Def.~\ref{def:acn}, which in the worst case take time linear in $|z|$. Hence, the overall time complexity is $O(|z|^2)$. Nevertheless, this procedure can be accelerated by a 
bisection search. Instead of checking items 1 and 2 in Def.~\ref{def:acn} step-by-step, one could start with the half-prefix of $z$ and apply bisection search recursively. Hence, we have the following result:
\begin{theorem}
Given a causal setting $\mathcal{C}= \lan z, \phi\ran$, such that $\Sigma \models \neg \phi$ and $\Sigma \models [z]\phi$. Suppose further that we have an oracle that can solve the projection problem in Eq.~\eqref{eq:proj} in $O(1)$ time, the problem of computing the cause under causal setting $\mathcal{C}$ can be solved in time $O(|z|\log |z|)$.
\end{theorem}

\section{Relation with Batusov and Soutchanski's Achievement Cause}
\label{sec:bscause}

Batusov and Soutchanski
\cite{batusov2018situation} provided an account of \emph{actual achievement cause} in the situation calculus. To compare the two definitions, we retrofit their account in the logic $\lang$. 

\begin{definition}[Actual Achievement Cause \cite{batusov2018situation}]
\label{def:bs18}
A causal setting $\mathcal{C} = \lan z, \phi \ran$ satisfies the achievement conditions of $\phi$ via the action-sequence pair $\lan a^\star,z^\star\ran$ if and only if there exists a pair $\lan a',z'\ran$ such that
\begin{enumerate}
    \item $\Sigma \models [z'] \neg \phi$;
    \item  for all $z''$ s.t. $z' \cdot a' \subseteq z'' \subseteq z$, $\Sigma \models [z''] \phi$;
    \item besides, it holds that (a) either $z^\star=z'$ and $a^\star = a'$; (b) or causal setting $\lan \Sigma, z', \mathfrak{R}[a',\phi] \land \poss(a')\ran$ 
        satisfies the achievement conditions via the action-sequence pair $\lan a^\star,z^\star\ran$.

\end{enumerate}
\end{definition}
\noindent Whenever a causal setting $\mathcal{C}$ satisfies the achievement condition via the pair $\lan a^\star, z^\star\ran$, the action $a^\star$ executed after $z^\star$ is said to be \emph{an achievement cause} in $\mathcal{C}$.

Intuitively, the action $a^\star$ executed after the sequence $z^\star$ is an achievement condition of $\phi$ in the causal setting $\mathcal{C} = \lan z, \phi \ran$ if it is a \emph{direct cause} (satisfying Item 1, Item 2, and Item 3(a) ), or it is a \emph{indirect cause} that enable the execution of a direct cause (satisfying Item 1, Item 2, and Item 3(b)). 
\cite{batusov2018situation} show that the achievement causes in a causal setting $\mathcal{C}$ forms a finite sequence of action-sequence pairs $\mathit{BSChain}(\mathcal{C})$, which is called the achievement causal chain of $\mathcal{C}$.

Given an action sequence $z$, we define $\mathit{Chain}(z)$ as $\mathit{Chain}(z):=\{\lan a',z'\ran: z'\cdot a' \subseteq z \}$, namely, $\mathit{Chain}(z)$ is the set of actions and together with the corresponding prefixes up to it. 

\begin{theorem}
    Given a causal setting $\mathcal{C} = \lan z, \phi \ran$, let $\mathit{BSChain}(\mathcal{C})$ be achievement causal chain of $\mathcal{C}$ by Definition ~\ref{def:bs18}, $z'$ be the cause of $\phi$ wrt $\Sigma$ and $z$ as Definition~\ref{def:acn}, then $\mathit{BSChain}(\mathcal{C}) \subseteq \mathit{Chain}(z')$.
\end{theorem}

\begin{proof}
    Without loss of generality, we suppose $\mathit{BSChain}(\mathcal{C}) = \{\lan a_1,z_1\ran,\ldots, \lan a_n,z_n \ran\}.$ It suffices to show that $z_n \cdot a_n \subseteq z'$ as $\mathit{Chain}(z')$ contains all actions and prefixes pairs of $z'$.  Since $\lan a_n,z_n \ran$ is the last pair in $\mathit{BSChain}(\mathcal{C})$, action $a_n$ executed in $z_n$ is a direct cause of $\phi$, namely, we have: 1. $\Sigma\models [z_n] \neg \phi$; 2. for all $z''$ such that $z_n\cdot a_n \subseteq z'' \subseteq z$, $\Sigma \models [z''] \phi$.

    Supposing $z_n \cdot a_n \not \subseteq z'$, since  $z_n \cdot a_n \subseteq z$ and $z' \subseteq z$, we have $z' \subset z_n \cdot a_n$ (hence $z' \subseteq z_n$). By Def.~\ref{def:acn}, $\Sigma \models [z^\star]$ for all $z^\star$ s.t. $z' \subseteq z^\star \subseteq z$. Since $z' \subseteq z_n \subseteq z$, we have $\Sigma \models [z_n] \phi$ which contradicts with $\Sigma\models [z_n] \neg \phi$ as above. Hence, the hypothesis $z_n \cdot a_n \not \subseteq z'$ is wrong.  \qed   
\end{proof}

\begin{example}
    Consider a variant of the disjunctive goal in Example \ref{example:disjunctivecauseinactions}. Namely, let $\mathcal{C}=\lan  \Sigma,z,\phi_3 \ran$  be a causal setting where  $\Sigma$ is as Example \ref{example:blockdomain}, $z= \pickup(C) \cdot \drop(C) \cdot \pickup(D)$, and $\phi_3:=\Broken(C) \vee \Holding(D)$. The achievement causal chain is as   
$\mathit{BSChain}(\mathcal{C}) = \{\lan\pickup(C), \lan \ran\ran,\lan\drop(C),\pickup(C)\ran \}.$
This is because, after the action $\drop(C)$ is executed after $\pickup(C)$, $\phi$ always holds under $\Sigma$, hence it is a direct cause. Besides, the action $\pickup(C)$ executed initially is an indirect cause as it enables the action $\pickup(C)$. Meanwhile,  for our notion of counterfactual achievement cause (of $\mathcal{C}$) $z' = \pickup(C) \cdot \drop(C) \cdot \pickup(D)$, we have $\mathit{Chain}(z')= \mathit{BSChain}(\mathcal{C}) \cup \{ \lan \pickup(D), \pickup(C)\cdot \drop(C)\ran \}$.  
\qed
\end{example}

\section{Relation with HP Definition of Causality}
\label{sec:hp}

Batusov and Soutchanski
\cite{batusov2018situation} had made a formal comparison between their account of actual achievement cause and the HP account of actual causality which can be summarized as (in the language of $\lang$) for every (part of) cause in a causal model in the HP account, there exists a corresponding action-sequence pair that appear in the achievement casual chain of translated causal setting (corresponding to the HP casual model). However, as aforementioned, Batusov and Soutchanski's account is not based on counterfactual analysis. We informally compared our notion of cause with the HP \emph{modified} definition of actual causality \cite{halpern2015modification} in terms of counterfactual analysis, especially in handling preemption or disjunctive goals. 

The HP account of causality is based on the \emph{structural equation model}. Below, we informally review the HP account and refer interested readers to \cite{halpern2016actual}.  
A causal model $\M$ is a tuple $\lan \U, \V,  \{f_X\}_{X\in \V} \ran $ where $\U$ and $\V$ are disjoint sets of \emph{exogenous} and \emph{endogenous} variables representing external or independent factors and internal factors. The value of each endogenous variable $X \in \V$ is specified by a function $f_X$ that may depend on exogenous variables and on endogenous variables that precede  $X$ with respect to a fixed order on $\V$. Moreover, the dependencies among variables are acyclic. A \emph{context} $\myvec{u}$ assigns values to all variables in $\U$. The \emph{language} of the HP account consists of formulas of the form $X=x$ (value assignment), boolean connectives, and formulas of the form $[\myvec{Y} \xleftarrow{} \myvec{y}] \phi$ (read as after intervening to set variables $\myvec{Y}$ to values $\myvec{y}$, $\phi$ holds).
The truth of formulas is given as (boolean connectives are understood in the standard way):
\begin{itemize}
    \item $(\M, \myvec{u}) \models X = x$ if $f_X$ has values $x$ when setting $\U$ according to $\myvec{u}$.

    \item $(\M, \myvec{u}) \models [\myvec{Y} \xleftarrow{} \myvec{y}] \phi$ if $(\M_{\myvec{Y} \assign \myvec{y}}, \myvec{u}) \models \phi$ where $\M_{\myvec{Y} \assign \myvec{y}}$ is the intervened model obtained from $\M$ by replacing $f_Y$ for $Y\in \myvec{Y}$ with the trivial function $f_Y=y$.
\end{itemize}

\begin{definition}
A conjunction of events $\myvec{X} = \myvec{x}$ ($\myvec{X} \subseteq \V$) is an \emph{actual cause} in $(\M,\myvec{u})$ of a query $\phi$ if all the following holds:
\begin{enumerate}
    \item $(\M,\myvec{u}) \models \myvec{X} = \myvec{x}$ and $(\M,\myvec{u}) \models \phi$.
    \item There exists a set $\myvec{W}$ (disjoint with $\myvec{X}$) of variables in $\V$ with $(\M,\myvec{u}) \models \myvec{W} = \myvec{w}$ and a set of values $\myvec{x}'$ such that $(\M,\myvec{u}) \models [\myvec{X}\xleftarrow{}\myvec{x}',\myvec{W}\xleftarrow{}\myvec{w}]\neg \phi$.
    \item No proper sub-conjunction of $\myvec{X}=\myvec{x}$ satisfies 1,2.
\end{enumerate}    
\end{definition}
Depending on how contingencies are handled (Item 2), the HP account has multiple variants. The definition above is the latest version, i.e., the \emph{modified} HP definition \cite{halpern2015modification}, where we only consider those contingencies or counterfactuals where the un-intervented variables have their original values ($\myvec{W} \xleftarrow{}\myvec{w}$). This is the same in spirit as our notion of achievement cause in narrative: when considering the counterfactual that a given action prefix in a narrative is absent, we would still consider the remaining possible actions in the narrative rather than other action sequences. The tuple $\lan \myvec{W},\myvec{w},\myvec{x}'\ran$
is called a witness to the fact that $\myvec{X}=\myvec{x}$ is a cause of $\phi$ and each conjunct of $\myvec{X}=\myvec{x}$ is called \emph{part of the cause}.

\begin{example}
    Consider the well-known ``Forest Fire'' example from \cite{halpern2005causes,halpern2016actual} with three endogenous variables: $MD$ (match dropped by an arsonist), $L$ (lightning strike), and $FF$ (the forest is on fire). For the disjunctive case where either $MD=\true$ or $L=\true$ is sufficient to cause a forest fire, i.e., $L=\true$, the causal model $M$ consists of 1 equation:  $FF:= MD=\true \vee L=\true$. In the context $\myvec{u}$ where both $MD=\true$ and $L=\true$, the cause of $FF=\true$ is the conjunction of $MD=\true$ and $L=\true$ in the causal model $(\M,\myvec{u})$ with a trivial witness $\myvec{W}=\emptyset$. Neither $MD=\true$ nor $L=\true$ is a cause, they are but part of the cause.
    \qed
\end{example}

The HP account has many appealing points and has been successfully adapted in many ways, such as defining \emph{blame and responsibility} \cite{chockler2004responsibility}, accounting for \emph{evidence} in verification \cite{baier2021verification}, and model-checking \cite{chockler2008causes,beer2012explaining}. Yet, it is also criticized for its limited expressiveness \cite{hopkins2007causality}, such as being unable to handle quantified effects compared to ours. The most severe issue is that the notion of cause there highly depends on how one models the domain, namely, the cause of a domain in a model might no longer be a cause when modeling the domain in another way. It is open what the criteria are in modeling a causal setting. Below, we show some possible and ``reasonable'' HP modeling of the disjunctive goal examples in Example \ref{example:disjunctivecauseinactions}, and compare our notion of cause with the HP cause in terms of counterfactual analysis. 

A possible HP modeling $M$ of the goal $\phi_2:=\Holding(C) \vee \Holding(D)$ and narrative $z = \pickup(C) \cdot \pickup(D)$ in Example \ref{example:disjunctivecauseinactions} is just like the ``Forest Fire'' example above. Namely, using two endogenous variables $PC$ and $PD$ for the two actions $\pickup(C)$ and $\pickup(D)$ and one endogenous variable $GL$ for the goal. The equations are given as ``$GL:= \true \textrm{ if } PC = \true \vee PD=\true \textrm{ otherwise } \false$''.  We will further assume that two exogenous variables enable $PC$ and $PD$, and omit the trivial enabling equations $f_{PC}$ and $f_{PD}$. It is not hard to see that under such model $\M$ and the context $\myvec{u}$ where both two actions are enabled, the cause of $GL$ is the conjunction of $PC = \true$ and $PD=\true$, just as the ``Forest Fire'' example, and the cause corresponds to our notion of achievement cause as in Example~\ref{example:disjunctivecauseinactions}. This is because the above modeling disregards the temporal order of occurrences $PC$ and $PD$, just as our notion of cause does. 

Indeed, it is possible to add the temporal relation of competing events in the HP modeling by adding extra endogenous variables. E.g. one can add two endogenous variables $PCG$, $PDG$ representing $PC$ and $PD$ causes the goal respectively, besides, add an equation ``$PDG := \true \textrm{ if } PCG=\false \land PD = \true \textrm{ otherwise } \false$'' to the model, i.e. $PD$ cause $GL$ only if $PC$ does not cause $GL$ and $PD$ happens. In this way, the HP account will indeed ascribe the cause of $GL$ to $PC$. Nevertheless, there are problems. First, it is unclear how many such variables are needed. For the above examples, where the events $\pickup(C)$ and $\pickup(D)$ are atomic in achieving $\phi_2$, two variables suffice. Things become complicated when events are not atomic. In such cases, ground actions of competing events might interleave; hence it is unclear how to identify these events. Another problem lies in the meaning and relations among these additional variables. In the above example, $PCG$ and $PDG$ seem to hardcode the causal relation among competing events and goals (seems cheating in some sense), so it is unsurprising to derive $PC$ as the cause. Lastly, it is unclear what the equations look like. For non-atomic events, actions might depend on each other; a natural thought is to specify $f_{X_a}$ for variable $X_a$ (supposing it corresponds to action $a$) by using variables of actions that $a$ depends on. However, this is also problematic as equations in HP models encode not just \emph{dependency} (or conditionals) but also causation. Namely, the equation $FF:=MD=\true \vee L = \true$ no just says $FF$ \emph{depends on} $MD$ and $L$, but also, $MD=\true$ or $L = \true$ will \emph{cause} $FF=\true$. For the sequence $\pickup(C) \cdot \drop(C)$, $\drop(C)$ might depend on $\pickup(C))$, but it does not mean that the occurrence of $\pickup(C)$ will cause the occurrence of $\drop(C)$. The relation between dependency (conditionals) and causality has remained open in philosophy for a long time \cite{mackie1965causes,gunther2019learning,andreas2020causation}. 

All of this is to say that the comparison between the HP formalism and our proposal in terms of modeling capabilities is not exact. While it may be possible to pick one particular model from one formalism and translate it to a model in another formalism in a faithful way, there are multiple considerations that need to be taken into account, especially in the presence of durative actions and competing actions.

\paragraph{Limitation of the Proposed Achievement Cause} 
Compared to the HP account, where the actual cause depends on how one models the domain, our methods have uniform standard modeling: the basic action theory. Yet, this limits the applicability of the notion of counterfactual achievement cause. In particular, the proposed definition suffers in handling competing non-atomic events where actions of different events might be non-consecutive and interleave. 


\begin{example}   

Consider $\phi_4:=\Broken(C) \vee \Broken(D)$ and $z= \pickup(C) \cdot \drop(C) \cdot \pickup(D) \cdot \quench(D) \cdot \drop(D)$ in the basic action theory $\Sigma$ of example \ref{example:blockdomain}. It is easy to check that $\Sigma \models \exec(z) \land [z]\phi_4$. Here, the achievement cause is $z' = \pickup(C) \cdot \drop(C) \cdot \pickup(D)$: in the contingency of $z'$ being absent, the remaining action sequence $\quench(D) \cdot \drop(D)$ is impossible, hence $\filter(z\backslash z') = \lan \ran$ under which $\Sigma$ does not entails $\phi_4$. It seems counter-intuitive as the cause is not $\pickup(C) \cdot \drop(C)$ and contains one \emph{redundant} action $\pickup(D)$. Yet, it is not in the sense of counterfactual analysis: the redundant action $\pickup(D)$ serves to disable the contingency of $\pickup(D) \cdot \quench(D) \cdot \drop(D)$ under which $\phi_4$ still holds in $\Sigma$ as $\Broken(D)$ becomes true. However, in the contingency of $\pickup(C) \cdot \drop(C) \cdot \pickup(D)$ being absent, the remaining sequence $\quench(D) \cdot \drop(D)$ is no longer executable under $\Sigma$. \qed
\end{example}

In the general case of multiple competing events (say, sequences of actions) to achieve the goal, the proposed notion of causes will include all the events (except the last one) and part of the last event. The fact that there are issues here is not surprising. This is precisely why the topic of causation remains an important and active area of philosophy, computer science, and artificial intelligence. What we have provided here is a first-order account that seems to have many reasonable features in an action language.

\section{Related Work}
\label{sec:relatedwork}

The most relevant work is by \cite{batusov2018situation}, who formalize causality in the situation calculus, improving on \cite{hopkins2005actual,hopkins2007causality} by handling quantified effects and preconditions. \cite{khan2021knowing} extends this to an epistemic modality, but not counterfactually. A counterfactual account is given by \cite{khan2020necessary}, parallel to \cite{batusov2018situation} but with a complex recursive definition.


The HP definition of actual cause has seen wide adoption and extensions, including blame and responsibility \cite{chockler2004responsibility}, verification \cite{baier2021verification}, databases \cite{meliou2010causality}, and model checking \cite{chockler2008causes,beer2012explaining}. Recent adaptations include mathematical explanations \cite{halpern2023mathematical} and image classifier explanations \cite{chockler2024explaining}. Still, it faces criticism for limited expressiveness \cite{hopkins2007causality} and dependence on domain modeling choices \cite{halpern2011actual}. There are many works that try to lift the expressiveness. E.g., \cite{GladyshevADD025} tried to capture temporal causal reasoning based on the CPLTL logic, i.e., causal LTL with past, in non-recursive structural equation models; \cite{LimaL24} proposed a propositional causal modal logic based on knowledge base semantics, interestingly, they showed that it is possible to capture the HP actual cause of a structural equation model in the language; \cite{BarberoS21} studied team semantics for interventionist counterfactuals and causal dependencies. Nevertheless, all these works are essentially propositional. In contrast, our account builds upon a first-order modal logic of actions and changes. Alternative approaches for accounting causes include Bochman’s logical models \cite{bochman2003logic,bochman2018actual,bochman2018laws}, grounded in the NESS condition \cite{wright1985causation}, a refinement of Mackie’s INUS condition \cite{mackie1965causes}. Unlike these, Batusov and Soutchanski’s account (and ours) is not reducible to any INUS-style condition and can handle quantified effects, which Bochman's cannot.
Probabilistic views of actual causation also exist, both philosophical \cite{lewis1986postscripts,kvart2004causation,fenton2017proposed} and formal, such as Pearl’s early causal beams \cite{pearl1998definition,pearl2009causality} and causal probabilistic logic \cite{vennekens2009cp,beckers2016general}.

Lastly, besides the situation calculus, there is another line of work on formalisms of actions and agency, the ``see to it that'' logic \cite{facingfuture2001,BalbianiHT08}, i.e., STIT. While the STIT family focuses on agents and choices, the situation calculus treats actions as first-class citizens; hence, it is more suitable for AI applications like planning and decision-making. Recent trends show that the studies of these two communities are starting to converge. E.g., \cite{HORTYPACUIT2017} enriches the STIT logic with action types in order to capture important concepts such as the epistemic sense of ability. Hence, it is interesting to see how a breakthrough from one community can benefit the other. For instance, there are works that study counterfactual emotions \cite{LoriniS11} and responsibility \cite{LoriniLM14} in the STIT logic.

\section{Conclusion}
\label{sec:conclusion}

In this paper, we propose a notion of counterfactual cause in the situation calculus. In the context that a given action history leads to a goal, we show that our notion of cause generalizes naturally to a notion of achievement cause. We analyze the relationship between our notion of the achievement cause and existing works. Future works include extending the proposal to an epistemic setting, just like \cite{khan2021knowing}. Besides, it is promising to investigate how our method can be adapted for robot programming in Golog and other applications that are being built on action languages and require actual causation.

\section*{Acknowledgement}
Daxin Liu is funded by the National Natural Science Foundation of China (NSFC No. 62506152) and the Fundamental and Interdisciplinary Disciplines Breakthrough Plan of the Ministry of Education of China (No. JYB2025XDXM118); Vaishak Belle is funded by a Royal Society University Research Fellowship at the University of Edinburgh.
\nocite{*}
\bibliographystyle{eptcs}
\bibliography{main}
\end{document}